\documentclass[a4paper,11pt]{article}
\usepackage[utf8]{inputenc}
\usepackage[T1]{fontenc}
\usepackage[margin=1in]{geometry}

\usepackage[colorlinks,linkcolor=blue,citecolor=blue]{hyperref}
\usepackage{amsfonts,amsmath,amsthm,amssymb}
\usepackage{mathtools}

\usepackage[round]{natbib}
\newcommand*\samethanks[1][\value{footnote}]{\footnotemark[#1]}

\title{Better Algorithms for Stochastic Bandits \\ with Adversarial Corruptions}
\author{
 Anupam Gupta%
 \thanks{Department of Computer Science, Carnegie Mellon University, Pittsburgh PA, USA; \texttt{anupamg@cs.cmu.edu}. Supported in part by NSF awards CCF-1536002, CCF-1540541, and CCF-1617790, and the Indo-US Joint Center for Algorithms Under Uncertainty. } 
 \and
 Tomer Koren%
 \thanks{Google Brain, Mountain View CA, USA; \texttt{\{tkoren,kunal\}@google.com}.}
 \and
 Kunal Talwar%
 \samethanks{}
}

\usepackage{amsthm}
\usepackage{thmtools}
\usepackage[capitalise,nameinlink]{cleveref}

\declaretheorem[style=theorem,numberwithin=,name=Theorem]{thm}

\declaretheorem[style=theorem,sibling=thm,name=Lemma]{lem}

\declaretheorem[style=theorem,numbered=no,name=Theorem]{thm*}
\declaretheorem[style=theorem,numbered=no,name=Lemma]{lem*}
\declaretheorem[style=theorem,numbered=no,name=Corollary]{cor*}
\declaretheorem[style=theorem,numbered=no,name=Proposition]{proposition*}
\declaretheorem[style=theorem,numbered=no,name=Claim]{claim*}
\declaretheorem[style=theorem,numbered=no,name=Fact]{fact*}
\declaretheorem[style=theorem,numbered=no,name=Observation]{observation*}
\declaretheorem[style=theorem,numbered=no,name=Conjecture]{conjecture*}

\declaretheorem[style=definition,numbered=no,name=Definition]{defn*}
\declaretheorem[style=definition,numbered=no,name=Remark]{remark*}
\declaretheorem[style=definition,numbered=no,name=Example]{example*}
\declaretheorem[style=definition,numbered=no,name=Question]{question*}

\usepackage{ktmacros}
\usepackage{nicefrac}
\usepackage{enumitem}

\renewcommand{\div}{\smash{\big/}}
\newcommand{\wt}[1]{\smash{\widetilde{#1}}}
\renewcommand{\Ex}{\mathbb{E}}

\newcommand{\rew}{R}
\newcommand{\regret}{\mathcal{R}}
\newcommand{\nrealized}{\tilde{n}}
\newcommand{\F}{\mathcal{F}}

\begin{document}
\maketitle

\begin{abstract}
We study the stochastic multi-armed bandits problem in the presence of adversarial corruption.
We present a new algorithm for this problem whose regret is nearly optimal, substantially improving upon previous work.
Our algorithm is agnostic to the level of adversarial contamination and
can tolerate a significant amount of corruption with virtually no degradation in performance.
\end{abstract}

\section{Introduction}
\label{sec:introduction}

We consider the stochastic multi-armed bandit problem with adversarial
corruptions. The classical stochastic multi-armed bandits problem~\citep{lai1985asymptotically}
has long been studied, with many algorithms including the Upper
Confidence Bound (UCB) algorithm~\citep{auer2002finite} and the Active
Arm Elimination (AAE) algorithm~\citep{even2006action}. (See,
e.g.,~\citealp{bubeck2012regret} for more details and many references.) Since the
assumption that the rewards of the arms are stochastic might appear quite rigid,
the \emph{adversarial} multi-armed bandit model has also been studied,
with algorithms like the {\sc Exp3} algorithm of~\cite{auer2002nonstochastic}.

The robustness of adversarial multi-armed bandit algorithms, however, comes at a heavy price: their guarantees are often significantly weaker than those of stochastic algorithms, the regret of which scales only logarithmically with the number of decision rounds.
A natural question is therefore: can we obtain results with guarantees that degrade smoothly as one goes from the stochastic setting towards the adversarial one?
More concretely, we can retain the favorable guarantees of the stochastic setting while tolerating a small amount of adversarial corruption?
These questions were the focus of the recent work by \cite{LykourisMPL18}, who studied the problem of stochastic multi-armed bandit with adversarial corruptions.
They motivated it by a variety of practical scenarios: these range from click fraud in pay-per-click advertisements, to the presence of spam and malicious reviews in recommendation systems.

The model of \citet{LykourisMPL18} is a variant of the classic stochastic multi-armed bandit problem, where each pull of an arm generates a stochastic reward that may be contaminated by an adversary before it is revealed to the player.
The difficulty of the problem then scales with the total amount of corruption introduced by the adversary.
\citet{LykourisMPL18} point out that standard stochastic algorithms like UCB and AAE can suffer severely from corruptions because the adversary, using a small amount of corruption, can cause the actual best arm to be eliminated by the algorithm, and hence make it fail miserably.
To address this, they show how to add robustness to these algorithms at the cost of moderate degradation in performance.

\subsection{The model}

Before proceeding, let us describe the setup in some more detail.
A player is faced with $K$ arms, indexed by $[K] := \{1,2,\ldots,K\}$.
Each arm $i \in [K]$ is associated with a reward distribution, unknown to the player, with mean $\mu_i$.
The best arm $i^\star = \argmax_i \mu_i$ is the one that maximizes the expected reward;
let $\mu^\star = \mu_{i^\star}$ denote the maximal expected reward, and for any suboptimal arm~$i \neq i^\star$ consider its gap $\Delta_i = \mu^\star - \mu_i \in [0,1]$.
The player interacts with the arms and has to repeatedly decide on an arm to pull over the course of~$T$ decision rounds.
This interaction is contaminated by an adversary, that may corrupt the feedback observed by the player on at most $C$ of these $T$ rounds.
(The eventual definition of $C$ will be a bit more general.)
After each arm pull, the player observes the possibly corrupted reward corresponding to that arm, and nothing else; this is traditionally called ``bandit feedback.''
The player's goal is to minimize her regret, which is the difference between the total reward of the best arm and her own cumulative reward.

In this setting, the algorithm of \citet{LykourisMPL18} achieves a regret bound of the form
$
	\wt{O}\big( KC \sum_{i \neq i^\star}\! 1\div\Delta_i \big)
	.
$%
\footnote{Here and throughout, the $\wt{O}$ notation suppresses all dependence on logarithmic terms.}
That is, their bound is $O(KC)$ times worse than the standard
$\wt{O}\big(\! \sum_{i \neq i^\star}\! 1\div\Delta_i \big)$ bound
achievable in the uncorrupted case by algorithms for the stochastic case
such as UCB, yet is still poly-logarithmic in the number of rounds $T$.
Importantly, their algorithm is \emph{agnostic} to the level of corruption inflicted by the adversary and need not know the value of $C$ ahead of time.
It is based on a ``multi-layered'' scheme that automatically adapts to the amount of adversarial contamination: the first layer is the fastest but also the most susceptible to corruptions, whereas the last is the most robust but is also the slowest to learn.

\subsection{Our results}

We show how to get significantly better regret bounds for stochastic multi-armed bandits with adversarial corruptions.
We give a new algorithm for the problem that attains a regret bound which removes the multiplicative dependence on $C$ appearing in the existing bounds and trades it for an additive dependence.
Our algorithm is entirely agnostic to the corruption level $C$, and is also extremely simple.
Formally, our main theorem is the following:
\begin{thm*}[\cref{thm:main}, informal]
There exists an algorithm for stochastic MAB with adversarial corruptions (called BARBAR; see \cref{alg:barbar} below), the regret of which is with high probability at most
\[
	O(KC) + \wt{O}\Bigg( \sum_{i \neq i^\star} \frac{1}{\Delta_i} \Bigg)
	.
\]
\end{thm*}

To put the parameters in context, consider an instance with $K$ arms and
gaps $\Delta_i = \sqrt{K/T}$;
in the uncorrupted case, this is a prototypical hard instance where the minimax
regret is $\wt{\Theta}(\sqrt{KT})$ and is matched by both stochastic
algorithms, like UCB, and adversarial algorithms, like {\sc Exp3}.
(In fact, this instance was used to show the $\Omega(\sqrt{KT})$ lower bound for adversarial bandits by \citealp{auer2002nonstochastic}.)
In the corrupted case, however,
algorithms for the stochastic case might fail and suffer linear $\Omega(T)$ regret.
We are robust to corruptions up to the level $C = O(\sqrt{T/K})$, obtaining
essentially the same regret as in the uncorrupted case. In contrast, the
previous bounds from~\cite{LykourisMPL18} can only guarantee
$O(K^{1.5}C\sqrt{T})$ regret, and thus suffer from a degradation in
performance due to the corruptions already at level $C=O(1)$ while becoming vacuous when $C = \Omega(\sqrt{T/K})$.

In the uncorrupted case, where $C=0$, our result recovers the classical
bounds (achieved, e.g., by the UCB algorithm) up to a logarithmic factor.
We also show that under several different conditions, it is possible to
improve our regret bound and remove the factor of $K$ from its first term;
we leave as an open question whether this is possible to achieve in the general, agnostic case without any further assumptions.

\subsection{Related work}

Bridging between the stochastic and adversarial settings in multi-armed bandits, and in online learning more generally, has been a topic of significant interest in recent years.
Most research in this direction has focused on obtaining ``the best of both worlds'' guarantees \citep{bubeck2012best,seldin2014one,auer2016algorithm,seldin2017improved,zimmert2018optimal,zimmert2019beating}.
The goal there is to achieve the better of the bounds at the two extremes: the worst case $O(\sqrt{T})$-type bound on any problem instance, and the better $O(\log{T})$-type bound whenever the instance is actually stochastic.
In particular, algorithms of this kind cannot adapt to instances that are close to being stochastic, which are treated as fully adversarial.
The model we consider here aims to gradually interpolate between the two extremes, as one goes from the stochastic setting and injects a moderate amount of adversarial corruptions.

Adversarial contaminations similar to those considered here have been studied before in the context of bandit problems.
\citet{seldin2014one} and \citet{zimmert2018optimal} consider a ``moderately contaminated'' regime in which the adversarial corruptions do not reduce the gap $\min_{i \neq i^\star}\! \Delta_i$ by more than a constant factor at any point in time.
This regime of contamination is very restrictive and, for example, precludes virtually any form of corruption in the early stages on learning.
\citet{KapoorKK2018} study a closely related model, where each step is
contaminated independently with probability $\eta$.
They also study a contextual bandit model where the corruption is adversarial, subject to the constraint that for any prefix, at most an $\eta$ fraction of the steps are corrupted.
\citet{altschuler2018best} consider contaminated stochastic bandits in the closely related best-arm identification setting, where rounds of a stochastic bandits instance are contaminated uniformly at random by an ``outlier'' distribution, and the goal is to identify the best arm in spite of the corruptions.
This model of contamination is again weaker than the adaptive adversarial model we consider here, and their goal is also somewhat different than ours.

A somewhat different flavor of robustness in multi-armed bandits has been explored by \cite{bubeck2013bandits}.
Their ``Robust UCB'' algorithm can tolerate heavy-tailed distributions of rewards; namely, it does not require boundness or sub-Gaussianity of the rewards, and instead only needs them to have bounded variance (or any higher-order moment).
These ideas were later adapted to other bandit settings (e.g., \citealp{yu2018pure,shao2018almost}).

Finally, we acknowledge the long line of research in statistics~\citep{huber64} and statistical (batch) learning~\citep{Valiant85,KearnsLi88} on making algorithms robust to various kinds of adversarial perturbations. There
have been several recent advances in computationally efficient algorithms for
problems of this kind; see, e.g.,~\citet{diakonikolas2016robust,lai2016agnostic,charikar2017learning,diakonikolas2018learning,klivans2018efficient} and the references therein.

\section{Setup: Stochastic MAB with Adversarial Corruptions}
\label{sec:setup}

We consider a multi-armed bandit problem with $K$ arms, indexed by $[K]
= \{1,2,\ldots,K\}$.  Each arm $i \in [K]$ is associated with an unknown
reward distribution with mean $\mu_i$; for simplicity we assume that the
reward distributions are all supported in $[0,1]$,%
in which case $\mu_i \in [0,1]$.  We let
$i^\star = \argmax_i \mu_i$ denote the best arm, and let $\mu^\star =
\mu_{i^\star}$ be the maximal expected reward.
Finally, for any arm $i \neq i^\star$ we define the gap $\Delta_i =
\mu^\star - \mu_i \in [0,1]$; to simplify notation we assume that the
optimal arm is unique, in which case $\Delta_i > 0$ for all $i \neq
i^\star$.

A player has to repeatedly choose an arm to pull and observe its reward
as feedback, that may be contaminated by an adversary;  
without loss of generality, we assume that the adversary is deterministic.
Formally, the interaction proceeds as follows. On each decision round $t=1,\ldots,T$:
\begin{enumerate}[label=(\roman*)]
\item
A stochastic reward $\rew^t_i \in [0,1]$ is drawn for each arm $i$ from its reward distribution;
\item
The adversary observes the vector $\rew^t$ and generates corrupted rewards $\wt{\rew}^t \in [0,1]^K$;
\item
The player picks $i_t \in [K]$, possibly at random, and observes the corrupted reward $\wt{\rew}^t_{i_t}$.
\end{enumerate}

The player is evaluated using the standard metric in stochastic MAB
models of {\em pseudo-regret} (henceforth called simply {\em regret};
see remark below about other possible notions of regret), given by
\begin{align*}
	\regret(T)
	= \sum_{t=1}^T \big( \mu^\star - \mu_{i_t} \big)
	.
\end{align*}
We will be interested in bounding the player's regret in terms of the level of corruption introduced by the adversary.
To this end, we measure the level of corruption $C$ as:
\begin{align*}
	C = \sum_{t=1}^T \norm{\wt{\rew}^t - \rew^t}_\infty
	.
\end{align*}

Our main focus will be on the case where $C$ is unknown, and the player has to be \emph{agnostic} to the level of corruption.

\paragraph{Remarks.}
Notice that we allow the adversary to be adaptive, in the sense that the corruptions on round $t$ may be determined as a function of the player's past choices as well as of the stochastic rewards in the current and previous rounds.
(However, the corruption on round $t$ is independent of the choice of the player on the same round.)
As a consequence, the corruption level $C$ is a random variable that depends on the stochastic rewards as well as on the player's internal randomization; the bounds we provide in this paper will depend on the realized value of this random variable.

Also note that our definition of regret above is precisely the familiar pseudo-regret metric ubiquitous in stochastic MAB models.
Our results directly extend to several other notions of regret; we discuss this in detail in \cref{sec:discussion}.

\section{The BARBAR Algorithm}
\label{sec:algo}

At a high level, our algorithm is similar in spirit to the Active Arm Elimination (AAE) algorithm for stochastic bandits, but contains some crucial modifications that make it robust to corruptions.

The algorithm proceeds in epochs which increase exponentially in length: the $m^{th}$ epoch has length roughly $2^{2m}$.
At the beginning of each epoch $m$, the algorithm computes an empirical estimate $\Delta^{m-1}_i$ of the gap of each arm $i$ based on its pulls during the previous epoch;
namely, we take the difference between the reward-per-pull for arm $i$ and a lower confidence bound for the highest reward-per-pull of any arm in epoch $m-1$.
The algorithm uses these estimates to bias towards pulling the seemingly-better arms: each arm is pulled roughly $(\Delta^{m-1}_i)^{-2}$ times, but never more than~$2^{2m}$ times, during epoch $m$.
As a consequence, the algorithm never eliminates an arm permanently; it gives the seemingly-not-so-good arms some recourse by pulling them a small number of times in each epoch.
This, along with the fact that each epoch only uses information from the immediately preceding epoch, ensures that any corruption has a bounded impact on the algorithm's regret.

\begin{algorithm}[th]
	\caption{BARBAR: Bandit Algorithm with Robustness: Bad Arms get Recourse}
   \label{alg:barbar}
\begin{algorithmic}[1]
	\STATE {\bf Parameters:} confidence $\delta \in (0,1)$, time horizon $T$.
	\STATE Initialize $T_0 = 0$ and $\Delta_i^{0} = 1$ for all $i \in [K]$.
	\STATE Set $\lambda = 1024\ln(\frac{8K}{\delta}\log_2\!T)$.
   	\FOR{epochs $m=1,2,\ldots$}
		\STATE Set $n_i^{m} = \lambda (\Delta_i^{m-1})^{-2}$ for all $i \in [K]$.
		\STATE Set $N_m = \sum_{i=1}^K n_i^{m}$ and $T_m = T_{m-1} + N_{m}$.
		\FOR{$t=T_{m-1}+1$ {\bfseries to} $T_{m}$}
			\STATE choose an arm $i$ with probability $n_i^m \div N_m$ and pull it.
	    \ENDFOR
		\STATE Set $r_i^m = S_i / n_i^m$ where $S_i$ is the total reward from the pulls of arm $i$ in this epoch.
    \STATE Set $r^m_\star = \max_i \{r_i^m - \frac{1}{16} \Delta_i^{m-1}\}$
		\STATE Set $\Delta_i^m = \max\{ 2^{-m}, r^m_\star - r_i^m \}$
   	\ENDFOR
\end{algorithmic}
\end{algorithm}

Formally, the BARBAR ({\em Bandit Algorithm with Robustness: Bad Arms get Recourse}) algorithm is presented as \cref{alg:barbar}.
The following is our main result: an $\wt{O}\big( KC + \sum_{i \neq i^\star}\! 1 \div \Delta_i \big)$ regret bound for the BARBAR algorithm.

\begin{thm}
\label{thm:main}
With probability at least $1-\delta$, the regret of \cref{alg:barbar} is bounded by
\[
	O\Bigg(KC + \sum_{i \neq i^\star} \frac{\log{T}}{\Delta_i} \log\!\Big(\frac{K}{\delta}\log{T}\Big) \!\Bigg).
\]
\end{thm}

In \cref{sec:tightexample} we give an example where our algorithm suffers $\Omega(KC)$ regret, showing that our analysis is essentially tight.
To convert the bound of \cref{thm:main} into a bound on the expected regret of the algorithm, one can set $\delta = 1 / T$; then, with the remaining probability the regret is trivially bounded by $T$ and the expected contribution to the regret is $O(1)$.
Thus, the expected regret is bounded by
\begin{align*}
	O\Bigg(K\,\Ex[C] + \sum_{i \neq i^\star} \frac{1}{\Delta_i} \log^2(KT) \Bigg).
\end{align*}

We now turn to prove \cref{thm:main}.
We start with a simple observation.

\begin{lem} \label{lem:epochs}
The length $N_m$ of epoch $m$ satisfies
\[
\lambda 2^{2(m-1)} \leq N_m \leq K\lambda 2^{2(m-1)}.
\]
Moreover, the number of epochs $M$ is at most $\log_2 T$.
\end{lem}

\begin{proof}
The arm $i_m$ achieving the maximum in the definition of $r^m_\star$ has $r^m_\star
- r_{i_m}^m \leq 0$, so it has $\Delta_{i_m}^m = 2^{-m}$.
This immediately implies that $N_{m+1} \geq n_{i_m}^{m+1} =
\lambda (\Delta^{m}_{i_m})^{-2} = \lambda 2^{2m}$.
The upper bound follows from the fact that each $\Delta_i^m$ is at least $2^{-m}$, and so $n_i^{m+1} \leq \lambda 2^{2m}$ for each $i$.
  The bound on the number of epochs $M$ follows from the lower bound on $N_m$.
\end{proof}

Let $C_m^i$ be the random variable denoting the sum of the corruptions in epoch $m$ to arm $i$'s
reward, and let $C_m := \max_i C_m^i$. Let $\nrealized_i^m$ be the
random variable denoting the actual number of pulls of arm $i$ in epoch $m$.
We define an event $\mathcal{E}$ that captures the relevant tail bounds
as follows:
\begin{align} \label{eq:event}
	\mathcal{E}
	:= \bigg\{
	\forall ~m,i
	\;:\;
	|r_i^m - \mu_i| \leq \frac{2C_m}{N_m} + \frac{\Delta_i^{m-1}}{16}
	\mbox{~~and~~} \nrealized_i^m \leq 2 n_i^m
	\bigg\}
	.
\end{align}

\begin{lem} \label{lem:reward_conc}
It holds that $\Pr[\mathcal{E}] \geq 1 - \delta$.
\end{lem}

The lemma is obtained from the following concentration inequalities concerning the reward estimates~$r^m_i$ computed by the algorithm and the number of arm pulls~$\nrealized^m_i$ it uses.

\begin{lem}
  \label{lem:reward_conc_part1}
  For any fixed $i, m$ and any $\beta \geq 4e^{-\lambda/16}$, we have
  \begin{align*}
    \Pr\!\bigg[|r_i^m - \mu_i| \geq \frac{2C_m}{N_m} +
    \sqrt{\frac{4\ln \nicefrac 4 \beta}{n_i^m}} \bigg]
    \leq \beta
    \mbox{~~~and~~~}
    \Pr\!\big[ \nrealized_i^m \geq 2 n_i^m \big] \leq \beta
    .
  \end{align*}
\end{lem}

\begin{proof}
  For this proof, we condition on all random variables before epoch $m$,
  so that $n^m_i$, $N_m$ and $T_{m-1}$ are deterministic quantities. At
  each step in this epoch, we pick arm $i$ with probability $q_i^m :=
  n_i^m \div N_m$.  Let $Y_i^t$ be an indicator for arm $i$ being pulled
  in step $t$. Let $\rew_i^t$ be the stochastic reward of arm $i$ on
  step $t$ and $c_i^t := \wt{\rew}_i^t - R_i^t$ be the corruption added
  to this arm by the adversary in this step. Note that $c_i^t$ may
  depend on all the stochastic rewards up to (and including) round~$t$, and also on all \emph{previous}
  choices of the algorithm (though not the choice at step $t$).
  Finally, if we denote $E_m = [T_{m-1}+1, \ldots, T_m]$ be the $N_m$
  many time-steps in epoch $m$, then the first random variable we aim to control can be expressed as
  \[ r_i^m = \frac{1}{n_i^m} \sum_{t \in E_m} Y_i^t \,
  \big( \rew_i^t + c_i^t \big). \]
  For ease of analysis, let us break the sum above into two, and define
  \[ A_i^m = \sum_{t \in E_m} Y_i^t \, \rew_i^t , \qquad \qquad B_i^m =
  \sum_{t \in E_m} Y_i^t \, c_i^t. \]

  Let us first bound the deviation of $A_i^m$.
  Observe that $\rew_i^t$ is an
  independent draw from a $[0,1]$-valued r.v.~with mean $\mu_i$ and
  $Y_i^t$ is a independent draw from an $\{0,1\}$-valued r.v.~with mean
  $q_i^m$. Moreover, we have that $\Ex[A_i^m] = N_m\cdot (q_i^m \mu_i) = n_i^m \mu_i \leq n_i^m$.
  Hence, by a Chernoff-Hoeffding bound (a multiplicative version thereof; see \cref{thm:chernoff} in \cref{sec:freedman}),
  we have
  \begin{gather}
    \Pr\!\bigg[ \bigg| \frac{A_i^m}{n_i^m} - \mu_i \bigg| \geq \sqrt{\frac{3\ln\nicefrac 4 \beta}{n_i^m}} \bigg]
    \leq \frac \beta 2
    . \label{eq:1}
  \end{gather}

  Next we turn to bound the deviation of $B_i^m$.
  Consider the sequence of r.v.s $X_1,\ldots,X_T$ defined by $X_t = (Y_i^t - q_i^m)\cdot c_i^t$ for all $t$.
  Then $\{X_t\}_{t=1}^T$ is a martingale difference sequence with
  respect to the filtration $\{\F_t\}_{t=1}^T$ generated by the r.v.s $\{Y_j^s\}_{j \in [K], s \leq t}, \{\rew_j^s\}_{j \in [K], s \leq t+1}$.
  Indeed, since the corruption $c_i^t$ becomes a deterministic value when
  conditioned on $\F_{t-1}$ (as we assume a deterministic adversary), and since $\Ex[Y_i^t \mid \F_{t-1}] = q_i^m$, we have
  \[ \Ex[X_t \mid \F_{t-1}] = \Ex[ Y_i^t - q_i^m \mid \F_{t-1}] \cdot c_i^t = 0. \]
  Further, we have
  $|X_t| \leq 1$ for all $t$, and we can bound the predictable quadratic variation of this martingale as:
  \[ V = \sum_{t \in E_m} \Ex[ X_t^2 \mid \F_{t-1} ]
  	\leq \sum_{t \in E_m} |c_i^t| \, \mathrm{Var}(Y_i^t)
    \leq q_i^m \sum_{t \in E_m} |c_i^t|.
  \]
  Applying a Freedman-type concentration inequality for martingales (see \cref{thm:Freedman} in \cref{sec:freedman}), we obtain that except with probability $\beta/4$,
  \begin{align*}
    \frac{B_i^m}{n_i^m}
    \leq \frac{q_i^m}{n_i^m} \sum_{t \in E_m} c_i^t + \frac{V + \ln \nicefrac 4 \beta}{n_i^m}
    \leq 2\frac{q_i^m}{n_i^m} \sum_{t \in E_m} |c_i^t| + \frac{\ln \nicefrac 4 \beta}{n_i^m}
    .
  \end{align*}
  Since $q_i^m = n_i^m \div N_m$ and $\sum_{t \in E_m} |c_i^t| \leq C_m$, and further $n_i^m \geq \lambda \geq 16\ln(4/\beta)$, we have with the same probability that
  \begin{align*}
    \frac{B_i^m}{n_i^m}
    \leq \frac{2C_m}{N_m} + \sqrt{\frac{\ln \nicefrac 4 \beta}{16 n_i^m}}
    .
  \end{align*}
  Similar arguments show that $-B_i^m/n_i^m$ satisfies this bound except with probability $\beta/4$.
  It follows that
  \begin{gather}
    \Pr\bigg[ \bigg| \frac{B_i^m}{n_i^m}\bigg| \geq \frac{2C_m}{N_m} + \sqrt{\frac{\ln \nicefrac 4 \beta}{16 n_i^m}} \bigg]
    \leq \frac \beta 2
    . \label{eq:2}
  \end{gather}
  We can now combine \cref{eq:1,eq:2} using a union bound
  to get the claimed bound, conditioned on the values of $n_i^m, N_m, T_m$.
  Finally, since the bound holds for any realizations of these r.v.s, it also holds unconditionally.

  For the second claimed bound, we use a Chernoff-Hoeffding bound (\cref{thm:chernoff} in \cref{sec:freedman}) on the random variable $\nrealized_i^m = \sum_{t \in E_m} Y_i^t$, when we again condition on past epochs and later remove the conditioning.
  The expected value is $\Ex[\nrealized_i^m] = n_i^m \geq \lambda$, and the probability that the r.v.~is more than twice its expectation is at most
  $2\exp(-n_i^m/3) \leq 2\exp(-\lambda/3) \leq \beta$.
  This completes the proof.
\end{proof}

\begin{proof}[Proof of \cref{lem:reward_conc}]
  We use \cref{lem:reward_conc_part1} for each arm $i$ and epoch
  $m$ with $\beta = \delta \div (2K \log_2{T})$; this value satisfies $\beta \geq 4 e^{-\lambda/16}$.
  Since $n_i^m = \lambda (\Delta_i^{m-1})^{-2}$ and $\ln(4/\beta) =
  \lambda \div 1024 = \lambda \div (4 \cdot 16^2)$, the lemma implies
  \[
  	\Pr\!\bigg[|r_i^m - \mu_i| \geq \frac{2C_m}{N_m} + \frac{\Delta_i^{m-1}}{16}
	\mbox{~~and~~}
	\nrealized_i^m \geq 2 n_i^m
	\bigg]
    \leq
    \frac{\delta}{K\log_2 T}.
  \]
  A union bound over the $K$ arms and the $\log_2 T$ epochs concludes the proof.
\end{proof}

For the rest of the analysis, we will condition on the event $\mathcal{E}$.
Our next step is to establish upper and lower bounds on our estimate of the maximal reward $r_{\star}^m$ in epoch $m$.
\begin{lem} \label{lem:R_bounds}
Suppose that $\mathcal{E}$ occurs.
Then for all epochs $m$,
\[
    - \frac{2C_m}{N_m} - \frac{\Delta_{i^\star}^{m-1}}{8}
    \leq r_{\star}^m - \mu^\star
    \leq \frac{2C_m}{N_m}.
\]
\end{lem}
\begin{proof}
Recall that $r_\star^m = \max_i \{r_i^m - \frac{1}{16} \Delta_i^{m-1}\}$.
In particular, $r_\star^m \geq r_{i^\star}^m - \frac{1}{16} \Delta_{i^\star}^{m-1}$; plugging in the lower bound on $r_{i^\star}^m$ implied by the event $\mathcal{E}$, we get the left inequality. On the other hand, plugging in the upper bound on $r_i^m$ yields \begin{align*}
  r_\star^m
  \leq \max_{i} \left\{ \mu_i + \frac{2C_m}{N_m} + \frac{\Delta_i^{m-1}}{16} - \frac{\Delta_i^{m-1}}{16} \right\}
  = \max_{i} \{\mu_i\} + \frac{2C_m}{N_m}
  = \mu^\star + \frac{2C_m}{N_m}
  .
  &\qedhere
\end{align*}
\end{proof}

We now establish upper and lower bounds on the gap estimate $\Delta_i^m$ of arm $i$ on epoch $m$.
For all epochs $m$, define the discounted corruption rate
\begin{gather} \label{eq:emmi}
  \rho_m
  := \sum_{s=1}^m \frac{2C_s}{8^{m-s}N_s}
  .
\end{gather}

\begin{lem} \label{lem:deltaji_ub}
Suppose that $\mathcal{E}$ occurs. Then for all epochs $m$ and arms $i$, it holds that
\begin{align*}
    \Delta_i^m
    &\leq 2\big( \Delta_i + 2^{-m} + \rho_m \big)
    .
\end{align*}
\end{lem}

\begin{proof}
  The proof is by induction on $m$. For $m=1$, the claim is trivially true for
  every $i$, because~$\Delta_i^m \leq 1 \leq 2\cdot2^{-m}$.
  Next, suppose that the claim holds for $m-1$.
  Using \cref{lem:R_bounds} and the definition of the event $\mathcal{E}$ (cf. \cref{eq:event}), we write
\begin{align*}
  r_\star^m - r_{i}^m
  &= (r_\star^m - \mu^\star) + (\mu^\star - \mu_i) + (\mu_i - r_i^m)
  \\
  &\leq \frac{2C_m}{N_m} + \Delta_i + \frac{2C_m}{N_m} + \frac{\Delta_i^{m-1}}{16}
  .
\end{align*}
Now using the induction hypothesis and the expression for $\rho_{m-1}$, we have
\begin{align*}
  r_\star^m - r_{i}^m
  &\leq \Delta_i + \frac{4C_m}{N_m} + \frac{1}{8}\Big(\Delta_i + 2^{-(m-1)} + \sum_{s=1}^{m-1} \frac{2C_s}{8^{m-1-s}N_s}\Big)
  \\
  &\leq 2\Big(\Delta_i + 2^{-m} + \sum_{s=1}^m \frac{2C_s}{8^{m-s}N_s}\Big)
  \\
  &= 2\big(\Delta_i + 2^{-m} + \rho_m \big)
  .
\end{align*}
Since $\Delta^{m}_i = \max\{ 2^{-m}, r_\star^m - r_{i}^m \}$, the claim follows.
\end{proof}

Armed with this result, we now prove a lower bound on $\Delta_i^m$.
\begin{lem}
  \label{lem:deltaji_lb}
  Suppose that $\mathcal{E}$ occurs.
  Then for all epochs $m$ and arms $i$, it holds that
  \begin{align*}
    \Delta_i^m
    &\geq \tfrac{1}{2} \Delta_i -  3 \rho_m - \tfrac{3}{4} 2^{-m}.
  \end{align*}
\end{lem}
\begin{proof}
We write
  \begin{align*}
    \Delta_i^m
    &\geq r_\star^m - r_i^m
    \\
    &\geq \Bigg( \mu^\star - \frac{2C_m}{N_m} - \frac{\Delta_{i^\star}^{m-1}}8 \Bigg)
      - \Bigg( \mu_i + \frac{2C_m}{N_m} + \frac{\Delta_i^{m-1}}{16} \Bigg),
    \intertext{where we use \cref{lem:R_bounds} to give a lower
    bound on $r_\star^m$ and \cref{eq:event} to give an upper bound on $r_i^m$.
    Since $\Delta_i = \mu^\star - \mu_i$, the above expression is at least}
    &\geq \Delta_i - \frac{4 C_m}{N_m} - \Bigg( \frac{\Delta_{i^\star}^{m-1}}{8} +
      \frac{\Delta_i^{m-1}}{16} \Bigg),
    \intertext{which using \cref{lem:deltaji_ub} can be further lower bounded by}
    &\geq \Delta_i - \frac{4 C_m}{N_m} - \Bigg(\frac{3}{8} \rho_{m-1} +
      \frac{3}{8} 2^{-(m-1)} + \frac{1}{8} \Delta_i \Bigg)
    \\
    &\geq \frac{\Delta_i}{2} - 3\sum_{s=1}^m \frac{2C_s}{8^{m-s}N_s} - \frac{3}{4} 2^{-m}
    .
  \end{align*}
Recalling the expression for $\rho_m$, the lemma follows.
\end{proof}

We are now ready to bound the regret of \cref{alg:barbar} and prove our main result.

\begin{proof}[Proof of \cref{thm:main}]
  We will prove that the claimed bound holds true under the event $\mathcal{E}$, which occurs with probability at least $1-\delta$.

  We decompose the total regret across epochs and within epochs, to the arms pulled. There are $\nrealized_i^m$ pulls of arm $i$ in epoch $m$, each causing a regret of $(\mu^\star - \mu_i)$. Thus the total regret can be written as
  \begin{gather}\label{eq:defreg}
    \sum_{m=1}^M \sum_{i=1}^K (\mu^\star - \mu_i)\nrealized_i^m
    \leq 2 \sum_{m=1}^M \sum_{i=1}^K (\mu^\star - \mu_i)n_i^m
    = 2 \sum_{m=1}^M \sum_{i=1}^K \Delta_i n_i^m
    .
  \end{gather}

  Fix an epoch $m$ and an arm $i$ and denote $\regret_i^m := \Delta_i n_i^m$.
  Recall that $n_i^m = \lambda (\Delta^{m-1}_i)^{-2},$ and consider three cases as follows.

  \paragraph{Case 1: $0 < \Delta_i \leq 4 \div 2^m$.}
  In this case we use the upper bound $n_i^m \leq \lambda 2^{2(m-1)}$, which follows from $\Delta^{m-1}_i \geq 2^{m-1}$.
  This yields
  \begin{align} \label{eq:case1}
  	\regret_i^m
	\leq \frac{4\lambda}{\Delta_i^2} \Delta_i
	= \frac{4\lambda}{\Delta_i}
	.
  \end{align}

  \paragraph{Case 2: $\Delta_i > 4 \div 2^m$ and $\rho_{m-1} < \Delta_i \div 32$.}
  In this case, \cref{lem:deltaji_lb} implies that
  \[
  	\Delta_i^{m-1}
	\geq \frac{1}{2} \Delta_i - 3 \rho_{m-1} - \frac{3}{4} 2^{-(m-1)}
	\geq \Delta_i \left( \frac{1}{2} - \frac{3}{32} - \frac{3}{8} \right)
	= \frac{1}{32} \Delta_i
	.
  \]
  In turn, we have
  $n_i^m = \lambda (\Delta_i^{m-1})^{-2} \leq 32^2 \lambda \div \Delta_i^2$
  from which it follows that
  \begin{align} \label{eq:case2}
  	\regret_i^m
	\leq \frac{32^2\lambda}{\Delta_i}
  	.
  \end{align}

  \paragraph{Case 3: $\rho_{m-1} \geq \Delta_i \div 32$ (and $\Delta_i > 4 \div 2^m$).}
  This rearranges to $\Delta_i \leq 32\rho_{m-1}$. We use again the upper bound
  $n_i^m \leq \lambda 2^{2(m-1)}$
  so that
  \begin{align} \label{eq:case3}
  	\regret_i^m
  	\leq \lambda \Delta_i 2^{2(m-1)}
	\leq 8 \lambda \rho_{m-1} 2^{2m}
	.
  \end{align}
  Considering the three cases, from \cref{eq:case1,eq:case2,eq:case3} we have for all arms $i \neq i^\star$ and epochs $m$ that
  \begin{align*}
  	\regret_i^m
	\leq \frac{32^2\lambda}{\Delta_i} + 8\lambda \rho_{m-1} 2^{2m}
	,
  \end{align*}
  and summing this over all epochs and arms, we have that the sum $\sum_{m=1}^M \sum_{i=1}^K \regret_i^m$ is bounded by
  \begin{align*}
  	32^2\lambda \sum_{i \neq i^\star} \frac{\log_2{T}}{\Delta_i} + 8\lambda \sum_{i \neq i^\star} \Big( \sum_{m=1}^M \rho_{m-1} 2^{2m} \Big)
	.
  \end{align*}
  For bounding the second term above, observe that
  \begin{align*}
    \sum_{m=1}^M \rho_{m-1} 2^{2m}
    &\leq \sum_{m=1}^M 2^{2m} \Big(\sum_{s=1}^{m-1} \frac{2C_s}{8^{m-1-s}N_s}\Big)
    \\
    &= 2\sum_{s=1}^M C_s \sum_{m=s}^M \frac{2^{2m}}{8^{m-1-s} N_s}
    ;
  \end{align*}
  Since $N_s \geq \lambda 2^{2(s-1)}$ (recall \cref{lem:epochs}), the inner sum is upper bounded by
  \[
  	\frac{16}{\lambda} \sum_{m=s}^M \frac{4^{m-1-s}}{8^{m-1-s}}
	\leq \frac{16}{\lambda} \sum_{j=1}^\infty 2^{-j}
	\leq \frac{16}{\lambda}
	.
  \]
	Combining the inequalities and plugging in the value of $\lambda$ gives the claimed regret bound.
\end{proof}

\section{Better Bounds for Special Cases}

In this section, we study a few special cases of the main model and show that under various kinds of additional assumptions, we can obtain an improved regret bound in which the first term $O(KC)$ is improved to $\wt{O}(C)$.
Some of these models are interesting in their own right; others serve to suggest that a better bound of the form $\wt{O}\big( C + \sum_{i \neq i^*}\! 1 \div \Delta_i \big)$ may be achievable in the general case.

\subsection{Corruption at a fixed unknown rate}

Consider the scenario each round is contaminated independently and uniformly at random by the adversary with probability $\eta \in (0,1)$, which may be unknown to the player. From the point of view of a stochastic bandits algorithm (such as Active Arm Elimination), this case is equivalent to operating over an instance where the means $\mu_i$'s are adversarially perturbed by $\pm \eta$. (This again can be seen from \cref{lem:reward_conc_part1} with $C_m = \eta N_m$.) A regret bound of
$\wt{O}\big(\eta T + \sum_{i \ne i^\star} 1 \div \Delta_i\big)$ then follows.

\subsection{Known corruption level}

Now suppose that the corruption amount $C$ is known to the player. In this case, consider the
classical Active Arm Elimination algorithm, with the modification that epoch
$m$ has length at least $N_m = \Omega(C + 2^m \log(K/\delta))$ which is ensured by pulling the winning arm sufficiently many times.
Even if all $C$ corruptions fall in epoch $m$, \cref{lem:reward_conc_part1} shows that the algorithm's empirical estimates of the means are accurate to within $\wt{O}(2^{-m})$.
It can be then verified that an arm with gap $\Delta_i \geq 2^{-m}$ will be eliminated before epoch $m+O(1)$ with high probability, leading to a bound of $\wt{O}\big( C + \sum_{i \neq i^*}\! 1 \div \Delta_i \big)$.

\subsection{Corruption on an unknown prefix}

Assume that the adversary's corruption is known to take place only in the first $C$ steps, where $C$ may be unknown.
In this case, the simple algorithm that plays AAE for steps $[2^m, 2^{m+1})$, starting afresh each time has regret $\wt{O}\big( C + \sum_{i \neq i^\star}\! 1 \div \Delta_i \big)$.

\subsection{Known minimal gap}

We can also obtain a similarly improved bound when the minimal gap $\Delta = \min_{i \neq i^\star} \Delta_i$ is known.
Consider the following variant of our algorithm: in each epoch $m$, play last epoch's winning arm~$2^{2m}$ times and play every other arm $\lambda\div\Delta^2$ times. This algorithm gets regret $K\div\Delta$ in each epoch from all arms except the winning one. If the winning arm is $i^\star$, it causes no regret in this epoch. For the winner to be an arm $i \neq i^\star$, the corruption in epoch $m-1$ must be at least $C_m \geq (\Delta_i - \Delta/8)2^{2(m-1)} \geq \Delta_i 2^{2(m-2)}$. This results in an additional regret of $O(\Delta_i 2^{2m})$. Thus, the overall regret is
$\wt{O}\big( C + K\lambda\div\Delta \big)$ for any corruption level $C$.

\subsection{Known maximal reward}

Finally, consider the case where the player knows the value of $\mu^\star$.
For this case, we change the algorithm so that
$\Delta^m_i = \max\{ 2^{-m}, \mu^\star - r^m_i, \Delta^{m-1}_i/2 \}$;
i.e., we use our knowledge of the optimal bias, and moreover our
estimates of the gaps does not drop very fast. In this case, we prove the following
variant of \cref{lem:deltaji_lb}:

\begin{lem}
  \label{lem:deltaji_lb_changed}
  Suppose that $\mathcal{E}$ occurs.
  Then for all epochs $m$ and arms $i$, it holds that
  \begin{align*}
    \Delta_i^m
    &\geq \frac{\Delta_i}{2} - \frac{2C_m}{N_m}.
  \end{align*}
\end{lem}

\begin{proof}
If $\Delta_i^{m-1} \geq \Delta_i$ then the change
  in the algorithm ensures that $\Delta_i^{m-1} \geq
  \Delta_i^{m} \div 2$. Otherwise, \cref{eq:event} ensures that $|r^m_i -
  \mu_i| \leq 2 C_m \div N_m + \Delta_i \div 16$. Now
  \[ \Delta_i^m \geq \mu^\star - r^m_i \geq (\mu^\star - \mu_i) -
    \frac{2C_m}{N_m} - \frac{\Delta_i}{16} \geq \frac{\Delta_i}{2} -
    \frac{2C_m}{N_m}. \]
  Hence the claim follows.
\end{proof}

To bound the regret, the proof is similar to Theorem~\ref{thm:main}.
Consider an arm $i$ and epoch $m$ such that
$C_m \div N_m \leq \Delta_i \div 8$. Lemma~\ref{lem:deltaji_lb_changed}
implies that $\Delta_i^m \geq \Delta_i \div 4$, and hence the regret due to
this arm $i$ in the next epoch $m+1$ is
$n^{m+1}_i \Delta_i = \lambda (\Delta_i^m)^{-2} \Delta_i \leq
O(\lambda \div \Delta_i)$. Summing this over all arms and epochs gives us
$\wt{O}\big(\sum_{i \neq i^\star}\! 1 \div \Delta_i\big)$.
Next, for arms $i$ and epochs $m$ where $C_m \div N_m \geq \Delta_i \div 8$,
the regret in the next round is
$n^{m+1}_i \Delta_i \leq 4n^{m}_i \Delta_i$, since we ensured that
$\Delta^m_i \geq \Delta^{m-1}_i \div 2$. Using the fact that
$\Delta_i \leq 8 C_m \div N_m$, we have the regret is at most
$32 n^m_i C_m \div N_m$. Summing over all the arms in this epoch
$m+1$ gives us $32 C_m$, and then summing over all rounds gives us
$32 C$. Hence, if we know $\mu^\star$, the total regret is at most
$ O(C) + \wt{O}\big( \!\sum_{i \neq i^\star}\! 1\div\Delta_i \big). $

\section{Discussion}
\label{sec:discussion}

We conclude with a discussion of various aspects of our model and results.

\subsection{A Lower bound}

We show an $\Omega(C)$ lower bound for the pseudo-regret of any algorithm.
Consider the case of two arms with deterministic rewards $0$ and $1$, and consider an adversary that at each time step swaps the rewards of the arms with probability $1/2$.
The expected total amount of corruption is therefore $T \div 2$, and it makes the arms appear indistinguishable to the algorithm.
Thus, with constant probability the algorithm will a suffer pseudo-regret of $\Omega(C)$. This lower bound for $C=T \div 2$ can be extended to smaller $C$ by placing this instance in the first $2C$ steps.
In other words, for the pseudo-regret under this contamination model there is a lower bound of $\Omega(C)$ that holds for all $C$ up to $T \div 2$. Moreover this lower bound holds even for the case of known $C$.

\subsection{Alternative notions of regret}

So far we have focused on the pseudo-regret metric, which is the most common in stochastic bandit problems.
One might consider several other possible notions of regret.
First, the realized regret with respect to the actual corrupted rewards may be of interest:
$$
	\regret'(T)
	= \max_{i \in [K]} \sum_{t=1}^T \big( \wt{\rew}^t_i - \wt{\rew}^t_{i_t} \big)
	.
$$
Under this metric, not only do the adversarial corruptions perturb the feedback observed by the player, but they also affect the rewards she gains.
As we discuss in \cref{sec:reg_v_pseudoreg}, any high-probability bound on the pseudo-regret also implies a high-probability bound on the realized regret, and our upper bounds transfer directly (up to a logarithmic factor) to this notion as well.

Note that for this notion of regret one could also attain an $O(\sqrt{T})$-type regret bound, at any level of corruption $C$, using an adversarial bandits algorithm like {\sc Exp3}.
On the other hand, \cite{LykourisMPL18} show a lower bound of $\Omega(C^{1-\delta})$ on the regret for any algorithm that gets regret $O(\log^{2-\eps}(T)/\Delta)$ for the uncorrupted case, for any $\epsilon,\delta>0$.
This lower bound still leaves open the possibility of obtaining a ``best of all worlds'' type guarantee that smoothly interpolates between the favorable stochastic guarantees and the robust adversarial results, which is meaningful even for $C = \Omega(\sqrt{T})$.
(For example, one might allow for slightly suboptimal $O(\mathrm{polylog}(T)/\Delta)$ regret in the non-corrupted case, which invalidates the latter lower bound.)

However, recalling the original motivation for our problem, the above realized regret is arguably not the ``correct'' notion.
In the motivating scenarios, such as click fraud or recommendation poisoning by fake reviewers, the rewards coming from the corruptions (the \emph{outliers}) are not of interest, and in fact, the goal is to perform well on the non-corrupted rewards (the \emph{inliers}).
The corresponding %
definition of regret for such cases would take into account only the non-corrupted steps:
\begin{align*}
	\regret'_{\mathrm{in}}(T)
	= \max_{i \in [K]} \sum_{t \notin \mathcal{C}} \big( \wt{\rew}^t_i - \wt{\rew}^t_{i_t} \big)
	,
\end{align*}
where $\mathcal{C}$ is the set of steps in which there was some corruption.%
\footnote{Here we assume a slightly more restrictive definition of the corruption level, where $C$ is the total number of rounds corrupted by the adversary. This definition is consistent with most of the motivating scenarios discussed.}
Our upper bounds, as well as the simple $\Omega(C)$ lower bound argued above, immediately extend to this notion as well.

As a consequence, this definition of regret does not admit $o(C)$-regret bounds for any $C$, and one cannot hope for a robust $O(\sqrt{T})$-type regret guarantee that would hold for any level of corruption.
How do we square this with the $O(\sqrt{T})$-type upper bounds that hold for the classical definition of regret, achieved by adversarial algorithms like {\sc Exp3}?
On the instance used to show the $\Omega(C)$ lower bound above, those algorithms will balance the linear regret on the inlier steps with a linear negative regret on the outlier steps.

\subsection{Open questions}

Our work leaves a gap between the lower bound of $\Omega(C)$ and the upper bound of $O(KC)$, closing which is a compelling open question.
Improving the $\log^2 T$ dependence in our expected regret bound to $\log{T}$ is another interesting problem for future investigation.
Finally, it would also be interesting to extend our results to other online prediction settings, e.g., linear bandits and contextual bandits.

\bibliography{robust-bandits}
\appendix

\section{Concentration Inequalities}
\label{sec:freedman}

In our analysis we require multiplicative versions of standard concentration inequalities, that provide sharper rates for random variables with small expectations.
First, we state a standard multiplicative variant of the Chernoff-Hoeffding bound.

\begin{thm}[e.g., {\citealt[Theorem 1.1]{dubhashi2009concentration}}]
\label{thm:chernoff}
Suppose that $X_1,\ldots,X_T$ are independent $[0,1]$-valued random variables, and let $X = \sum_{t=1}^T X_t$.
Then for any $\epsilon > 0$,
\begin{align*}
	\Pr\!\big[ |X - \Ex[X]| \geq \epsilon\, \Ex[X] \big]
	\leq 2\exp\left(-\frac{\epsilon^2}{3} \Ex[X]\right).
\end{align*}
Equivalently, for any $\delta>0$,
\begin{align*}
	\Pr\!\left[ |X - \Ex[X]| \leq \sqrt{3\Ex[X] \ln\frac 2 \delta} \right]
	\geq 1-\delta.
\end{align*}
\end{thm}

Next, we record a variant of Freedman's martingale concentration
inequality, taken from \citet{Beygelzimer11}. (Technically, the concentration inequality in~\citet{Beygelzimer11} is stated for the filtration generated by the martingale itself; the proof however extends straightforwardly to more general filtrations.)

\begin{thm}[{\citealt[Theorem 1]{Beygelzimer11}}]
\label{thm:Freedman}
Suppose that $X_1,\ldots,X_T$ is a martingale difference sequence with respect to a filtration~$\{\F_t\}_{t=1}^T$, and let $X = \sum_{t=1}^T X_t$.
Assume that $|X_t| \leq b$ for all $t$, and define
$
  V = \sum_{t=1}^T \Ex[X_t^2 \mid \F_{t-1}].
$
Then for any $\delta > 0$, %
\begin{align*}
  \Pr\!\left[ X \leq \frac{V}{b} + b \ln \frac 1 \delta
  \right] \geq 1-\delta.
\end{align*}
\end{thm}

\section{Tightness of Our Analysis}
\label{sec:tightexample}

Here we give an example where our algorithm suffers regret $\Omega(CK)$,
showing that the first term in our analysis is tight.
Consider the case where the rewards of the arms are $\mu_1 = 1$ and $\mu_i = 0$ for all
$i \neq 1$. In the uncorrupted case, for each epoch $m \geq 1$ we have
$\Delta_1^m = 2^{-m}$ and
$\Delta_i^m = r_\star^m = 1 - \frac{ 1 - (1/16)^m }{15} \in
[\nicefrac{14}{15}, \nicefrac{15}{16}]$. Consequently, the $m^{th}$
epoch has length $N_m = \lambda(2^{2(m-1)} + O(K))$ as long as we see no
corruptions.

Suppose $C = N_{c} = \lambda(2^{2(c-1)} + O(K))$ for some $c > 2\log_2 K$, say.
In our example, all the corruptions happen in epoch~$c$. Specifically, we corrupt
the high-value arm to also have payoff zero in all $N_c$ time-steps in
this epoch. This causes $r_\star^c \leq 0$ and all $\Delta_i^c = 2^{-c}$: in
essence we lose all information from the previous epochs. Therefore the
length of the next epoch becomes
$N_{c+1} = \lambda K 2^{2c} \approx 2KC$ which is $K$ times as much;
moreover, each arm is pulled $\lambda 2^{2c} \approx C$ times in this
epoch. The regret just from this epoch is $\approx (K-1)C$, which shows
we indeed need the factor of $K$ multiplying the $C$ term in the regret
for our algorithm.

\section{From Pseudo-regret to Regret}
\label{sec:reg_v_pseudoreg}

Here we show how a (high probability) bound over the pseudo regret can be converted to a bound into a (high probability) bound over the actual, realized regret in a black-box fashion.
We assume the regret is defined based on the uncorrupted stochastic rewards; the analogous notion that uses the corrupted rewards is only an additive $O(C)$ larger in the worst case.
As a consequence, our (positive) results regarding the pseudo regret immediately extend to the realized regret with respect to either the corrupted or uncorrupted rewards.

Using the notation established in \cref{sec:setup}, recall that the pseudo regret is given by
$$\regret = \sum_{t=1}^T \big( \mu^\star - \mu_{i_t} \big) = \sum_{i \neq i^\star} T_i \Delta_i,$$
where $T_i$ is the number of times arm $i$ was pulled.
On the other hand, the (realized) regret is the quantity
$$\regret' = \max_{i \in [K]} \sum_{t=1}^T \big( \rew^t_{i} - \rew^t_{i_t} \big).$$

Suppose that we have established a bound $\Pr( \regret \leq B ) \geq 1-\delta$ on the pseudo regret, where $B \geq \max_{i \neq i^\star} 1\div\Delta_i$. We want to establish a high-probability bound on $\regret'$. Define the quantity:
$$\regret^\star = \sum_{t=1}^T \big( \rew^t_{i^\star} - \rew^t_{i_t} \big).$$

We will separately bound $\regret' - \regret^\star$ and $\regret^\star - \regret$.
Applying Azuma's inequality on the random variable $\regret^\star-\regret$ we have with probability at least $1-\delta$ that
\begin{align*}
	\regret^\star-\regret
	\leq
	\sqrt{ \sum_{\smash{i \neq i^\star}}\, T_i \ln\frac{1}{\delta} }
	.
\end{align*}
Applying Holder's inequality,
\begin{align*}
	\regret^\star-\regret
	\leq
	\sqrt{\sum_{i \neq i^\star} T_i\Delta_i \cdot \frac{1}{\Delta_i} \ln\frac{1}{\delta} }
	=
 \sqrt{\Bigg(\sum_{i \neq i^\star} T_i\Delta_i \Bigg) \cdot \Bigg(\max_{i \neq i^\star}  \frac{1}{\Delta_i} \ln\frac{1}{\delta}\Bigg)}.
\end{align*}
Now, since $B$ is a bound over the pseudo-regret we have that $\regret = \sum_{i \neq i^\star} T_i \Delta_i \leq B$ with probability at least~$1-\delta$.
We have therefore shown that with probability at least $1-2\delta$,
\begin{align*}
	\regret^\star - \regret
	\leq
  \sqrt{B \max_{i \neq i^\star}  \frac{1}{\Delta_i} \ln\frac{1}{\delta}}
	.
\end{align*}

On the other hand, another application of Azuma's inequality shows that with probability $\geq 1-\delta$,
$$
	\regret' - \regret^\star
	\leq \max_{i \in [K]} \Bigg(\!-\Delta_i T + \sqrt{T \ln \frac{K}{\delta}}\Bigg)
	.
$$
If $\min_{i\neq i^\star} \Delta_i \geq \sqrt{\ln(K/\delta)/T}$, then this difference is negative.
Otherwise, it is bounded as
$$\regret' - \regret^\star \leq \sqrt{T \ln\frac{K}{\delta}}
\leq \max_{i \neq i^\star} \frac{1}{\Delta_i}\ln\frac{K}{\delta}.$$
Also, we have $\max_{i \neq i^\star} 1\div\Delta_i \leq B$ by assumption.
It follows that with probability at least $1-3\delta$,
$$
	\regret'
	\leq B\Bigg( 1 + \sqrt{\ln\frac{1}{\delta}} + \ln \frac{K}{\delta} \Bigg).
$$
Note also that if $B^{-1}\ln \frac{K}{\delta} \cdot \big(\!\max_{i \neq i^\star}\! 1\div\Delta_i\big)$ is bounded by a constant (as it typically is),
then $\regret'$ is at most $O(B)$.

\end{document}